%% file: smerf.tex
\documentclass[onefignum, onetabnum,onealgnum]{siamonline190516}
\input{nd_preamble.tex}

\title{Robust Similarity and Distance Learning via Decision Forests}

\author{
    Tyler M. Tomita%
    \thanks{Department of Psychological and Brain Sciences,
        Johns Hopkins University (\email{ttomita@jhu.edu})}
    \and
    Joshua T. Vogelstein%
    \thanks{Department of Biomedical Engineering,
        Institute for Computational Medicine,
        Kavli~Neuroscience~Discovery Institute,
        Johns Hopkins University (\email{jovo@jhu.edu},
        \url{http://www.jovo.me/}).}
}

\begin{document}

\maketitle



\noindent
\textbf{%
Canonical distances such as Euclidean distance often fail to capture the appropriate relationships between items, subsequently leading to subpar inference and prediction. Many algorithms have been proposed for automated learning of suitable distances, most of which employ linear methods to learn a global metric over the feature space. While such methods offer nice theoretical properties, interpretability, and computationally efficient means for implementing them, they are limited in expressive capacity. Methods which have been designed to improve expressiveness sacrifice one or more of the nice properties of the linear methods. To bridge this gap, we propose a highly expressive novel decision forest algorithm for the task of distance learning, which we call Similarity and Metric Random Forests (\Smerf). We show that the tree construction procedure in \Smerf~is a proper generalization of standard classification and regression trees. Thus, the mathematical driving forces of \Smerf~are examined via its direct connection to regression forests, for which theory has been developed. Its ability to approximate arbitrary distances and identify important features is empirically demonstrated on simulated data sets. Last, we demonstrate that it accurately predicts links in networks.  
}

\section{Introduction}
Many machine learning and data mining tasks rely on a good similarity or distance metric which captures the appropriate relationships between items. The most notable examples are k-nearest neighbors (k-NN) and k-means. Having an appropriate distance is crucial in many applications such as classification, regression, clustering, information retrieval, and recommender systems. It also plays an important role in psychological science and cognitive neuroscience. For instance, one common psychology experiment is to have subjects rate similarities/distances between many pairs of items. The researcher then attempts to understand how features of the items drive the distance judgements \cite{tversky1977,oswal2016}.

Motivated by the above, many researchers have demonstrated how learning an appropriate distance can improve performance on a variety of tasks. \cite{xing2003,goldberger2004, weinberger2009, bar2005, davis2007} propose various ways of learning a Mahalanobis distance metric to improve clustering and k-NN prediction. \cite{oswal2016} learns a bilinear similarity for inference tasks in psychology and cognitive neuroscience. All of these methods have nice properties such as global optimality guarantees and interpretability, but they are limited in expressive capacity. Thus, methods have been proposed to address this. The random forest distance (RFD) \cite{xiong2012} views distance metric learning as a pairwise classification problem -- items in a pair are similar or they are not. This method is really a data transformation method, employing standard classification forests on the transformed data. This data transformation is costly, as it doubles the number of features and squares the number of data points that need to be partitioned, potentially leading to very deep trees for large training sample sizes. \cite{liu2019} proposes a neural similarity learning method for improving visual recognition with CNNs. Various similarities based on siamese networks have been proposed for tasks such as signature verifcation \cite{bromley1994} and one-shot learning \cite{koch2015}.

While the more expressive methods may demonstrate high accuracy on complex tasks, they lose computational efficiency and/or interpretability. Furthermore, many are tailored for specific structured problems, which limits flexibility. To bridge the gap, we propose a robust, interpretable, and computationally efficient Similarity and Metric Random Forests (\Smerf) for learning distances. Given an observed set of pairwise distances, \Smerf~trees attempt to partition the points into disjoint regions such that the average pairwise distances between points within each region is minimal. Our method directly generalizes standard classification and regression trees. Like classification and regression forests, \Smerf~is embarrassingly parallelizable, and unlike RFD, \Smerf~partitions $n$ points in $p$ dimensions rather than $n^2$ points in $2p$ dimensions. We show that \Smerf~can approximate a range of different notions of distances. Last, we demonstrate its flexibility and real-world utility by using it to predict links in networks.

\section{Similarity and Metric Random Forests}
\label{sec: 2}
Suppose we observe a set of points $\{\mathbf{x}_i\}_1^n \in \Xc \subseteq \Real^p$, along with a symmetric matrix $\Zm \in \Real^{n \times n}$ whose $(i,j)^{th}$ element $z_{ij}$ represents some notion of dissimilarity or distance between $\mathbf{x}_i$ and $\mathbf{x}_j$. We wish to learn a function $g(\mathbf{x},\mathbf{x}'): \Real^p \times \Real^p \rightarrow \Real$ which predicts the distance $z$ for a new pair of observations $\mathbf{x}$ and $\mathbf{x}'$.\footnote{Typically a distance or dissimilarity $z$ is nonnegative, but \Smerf~can operate on negative-valued distances (for example, a distance which is defined by taking the negative of a nonnegative similarity measure.} To this end, we introduce an ensemble decision tree-based method called Similarity and Metric Random Forests (\Smerf) (technically, it learns a semi-pseudometric). Starting at the root node of a tree, which is the entire input space $\Xc$, the training observations are partitioned into disjoint regions of the input space via a series of recursive binary splits. The orientation and location of each split is found by maximizing the reduction in average pairwise distance of points in the resulting child nodes, relative to the parent node. For convenience, we assume splits are made orthogonal to the axes of the input space, although in practice we allow arbitrarily oriented splits. Let $\Sc$ be a set of points at a particular split node of a tree and $n_s = |\Sc|$. The average pairwise distance $\Id$ is:
\begin{align*}
\Id(\Sc) = \frac{1}{\ns^2}\sum_{i,j \in \Sc}\zij
\end{align*}

Let $\eta = (j, \tau)$ denote the tuple of split parameters at a split node, where $j$ indexes a dimension to split and $\tau$ specifies where to split along the $j^{th}$ dimension. Furthermore, let $\Sc^L_{\eta} = \{i: \mathbf{x}_i^{(j)} \leq \tau, \forall i \in \Sc\}$ and $\Sc^R_{\eta} = \{i: \mathbf{x}_i^{(j)} > \tau, \forall i \in \Sc\}$ be the subsets of $\Sc$ to the left and right of the splitting threshold, respectively. $x^{(j)}$ denotes the $j^{th}$ dimension of $\mathbf{x}$. Denote by $n_L$ and $n_R$ the number of observations in the left and right child nodes, respectively. Then a split is made via:
\begin{align}
\label{eq: opt-smerf}
    \textstyle \eta^* = \argmax\limits_{\eta}\; \ns \Id(\Sc) - n_L \Id(\Sc^L_{\eta}) - n_R \Id(\Sc^R_{\eta})
\end{align}
Eq.~(\ref{eq: opt-smerf}) finds the split that maximally reduces the average pairwise distance of points in the child nodes, relative to that of the parent node. This optimization is performed exhaustively. Nodes are recursively split until a stopping criterion has been reached, either a maximum depth or a minimum number of points in a node. The end result is a set of leaf nodes, which are disjoint regions of the feature space each containing one or more training points. An ensemble of $B$ randomized trees are constructed, where randomization occurs via the following two procedures: 1) using a random subsample or bootstrap of the training points for each tree and 2) restricting the search in Eq.~(\ref{eq: opt-smerf}) over a random subsample of the input feature dimensions. 

In order to predict the distance for a new pair of points, a particular notion of distances between all pairs of leaf nodes is computed, which is defined in the following.\footnote{The pairwise leaf node distance we adopt is one of many possible sensible distances.} Let $l_{b,a}$ be the $a^{th}$ leaf node of the $b^{th}$ tree, and let $\Sc(l_{b,a}) = \{i: \mathbf{x}_i \in l_{b,a} \forall i \in [n]\}$ be the subset of the training data contained in $l_{b,a}$. The distance between leaves $l_{b,a}$ and $l_{b,a'}$ is
\begin{align}
    \label{eq: treepredict}
    h(l_{b,a}, l_{b,a'}) = \frac{1}{|\Sc(l_{b,a})||\Sc(l_{b,a'})|}\sum_{i \in \Sc(l_{b,a})}\sum_{j \in \Sc(l_{b,a'})}\zij.
\end{align}

$\mathbf{x}$ and $\mathbf{x'}$ are passed down the tree until they fall into a leaf node. Their predicted distance is simply the distances of the leaves that $\mathbf{x}$ and $\mathbf{x'}$ fall into. Letting $l_b(\mathbf{x})$ and $l_b(\mathbf{x}')$ be the leaf nodes where $\mathbf{x}$ and $\mathbf{x'}$ fall into for the $b^{th}$ tree, the $b^{th}$ tree prediction is $h(l_b(\mathbf{x}),l_b(\mathbf{x}'))$. The prediction made by the ensemble of trees is the average of the individual tree predictions:
\begin{align}
    \label{eq: forestpredict}
g(\mathbf{x},\mathbf{x}') = \frac{1}{B} \sum_{b=1}^B g_b(\mathbf{x},\mathbf{x}') = \frac{1}{B} \sum_{b=1}^B h(l_b(\mathbf{x}),l_b(\mathbf{x}')) 
\end{align}

The utility of $\Smerf$ is diverse. For instance, $\Zm$ can represent distances between items for information retrieval, or it can represent links between nodes in a network. For knowledge discovery, one can use standard computationally efficient tree-based variable importance methods, enabling identification of features driving the observed distances/dissimilarities between points. In classification, regression, and clustering, one can learn distances to improve k-NN or k-means.

\section{\Smerf~Generalizes Classification and Regression Trees}
We show that the classification and regression tree procedures of Breiman et al \cite{breiman1984} are special cases of \Smerf~trees, instantiated by specifying particular notions of pairwise distances. Both classification and regression trees recursively split the training points by optimizing a split objective function. The classification tree procedure constructs a tree from points $\{\mathbf{x}_i\}_1^n \in \Xc \subseteq \Real^p$ and associated class labels $\{c_i\}_1^n \in \Cc = \{1,\ldots,K\}$. The Gini impurity $I_G$ for a tree node sample $\Sc$ is defined as
\begin{align*}
\Ig(\Sc) = 1 - \sum_{k=1}^K f_k^2,
\end{align*}
where $f_k = \frac{1}{\ns}\sum_{i \in \Sc} \II[c_i = k]$ is the fraction of points in $\Sc$ whose class label is $k$. A classification tree finds the optimal orientation and location to split the training points using the following optimization:
\begin{align}
    \label{eq: opt-class}
    \textstyle \eta^* = \argmax\limits_{\eta}\; \ns \Ig(\Sc) - n_L \Ig(\Sc^L_{\eta}) - n_R \Ig(\Sc^R_{\eta})
\end{align}
This equation is a special case of Eq.~(\ref{eq: opt-smerf}), when the pairwise distance $\zij$ is defined as the indicator of points $i$ and $j$ belonging to different classes.

\begin{prop}
    Let $\mathcal{T}_{class}$ be a classification tree constructed from $\{\mathbf{x}_i\}_1^n$ and class labels $\{c_i\}_1^n$ using Eq.~(\ref{eq: opt-class}). Let $\Zm$ be the pairwise distance matrix whose element $\zij = \II[c_i \neq c_j]$. The tree $\mathcal{T}_{\Smerf}$ constructed from $\{\mathbf{x}_i\}_1^n$ and $\Zm$ using Eq.~(\ref{eq: opt-smerf}) is equivalent to $\mathcal{T}_{class}$.
\end{prop}

Similarly, we claim a regression tree is a special case of a \Smerf~tree. A regression tree is constructed from $\{\mathbf{x}_i\}_1^n$ and associated continuous responses $\{y_i\}_1^n \in \Yc \subseteq \Real$. Denote by $\Iv(\Sc)$ the (biased) sample variance of the responses points $\Sc$. Specifically,
\begin{align*}
    \Iv(\Sc) = \frac{1}{\ns} \sum_{i \in \Sc} (y_i - \frac{1}{\ns} \sum_{i \in \Sc} y_i)^2
\end{align*}
A regression tree finds the best split parameters which maximally reduce the sample variance in the child nodes, relative to the parent node. Specifically, the optimization is
\begin{align}
    \label{eq: opt-reg}
    \textstyle \eta^* = \argmax\limits_{\eta}\; \ns \Iv(\Sc) - n_L \Iv(\Sc^L_{\eta}) - n_R \Iv(\Sc^R_{\eta})
\end{align}
This optimization is a special case of Eq.~(\ref{eq: opt-smerf}) when the distance of a pair of points is defined as one-half the squared difference in the responses of the points.

\begin{prop}
    Let $\mathcal{T}_{reg}$ be a regression tree constructed from $\{\mathbf{x}_i\}_1^n$ and responses $\{y_i\}_1^n$ using Eq.~(\ref{eq: opt-reg}). Let $\Zm$ be the pairwise distance matrix whose element $\zij = \frac{1}{2}(y_i - y_j)^2$. The tree $\mathcal{T}_{\Smerf}$ constructed from $\{\mathbf{x}_i\}_1^n$ and $\Zm$ using Eq.~(\ref{eq: opt-smerf}) is equivalent to $\mathcal{T}_{reg}$.
\end{prop}
Proofs for these propositions are in the appendix.

One implication of Propositions 1 and 2 is that different flavors of classification and regression trees can be constructed, simply by changing the notion of distance. Doing so is equivalent to changing the split objective function. For example, one can construct a more robust regression tree by defining the pairwise distance $\zij = |y_i - y_j|$.

\section{Examining \Smerf~Under a Statistical Learning Framework}
\label{sec: theory}
In this section, a statistical learning framework is developed, which will ultimately shed light on the mathematical driving forces of \Smerf. Suppose a pair of i.i.d. random vectors $X, X' \in [0,1]^p$ is observed, and the goal is to predict a random variable $Z \in \Real$ which represents some notion of distance between $X$ and $X'$. Formally, we wish to find a function $g(X,X')$ that minimizes $L(g) = \Exp[(g(X,X') - Z)^2]$. To this end, we assume a training sample $T_n = \{(X_i, X_j, Z_{ij}): i = 1,\ldots,n; j \geq i\}$ distributed as the prototype $(X,X',Z)$ is observed, where each $X_i$ is i.i.d.\footnote{$Z_{ij}$ and $Z_{ik}$ may be dependent, since they represent distances for two pairs which share the same sample $X_i$} Furthermore, assume both the distance and $g$ are symmetric, so that $Z_{ij} = Z_{ji}$ and $g(X,X') = g(X',X)$. The objective is to use $T_n$ to construct an estimate $g_n(X,X';T_n): [0,1]^p \times [0,1]^p \rightarrow \Real$ of the Bayes optimal distance function $\gb = \Exp[Z|X,X']$, which is the true but generally unknown minimizer of $L(g)$. For convenience in notation, we will omit $T_n$ from $g_n(X,X';T_n)$ when appropriate. An estimate $g_n(X,X')$ is said to be consistent if $g_n(X,X') \xrightarrow{P} \gb$.

We analyze consistency of procedures under a simple class of distributions over $(X,X',Z)$. Suppose there exists an additional response variable $Y \in \Real$ associated with each $X$, where $Y$ may be observed or latent within the training sample. This induces a distribution over the joint set of random variables $((X,Y),(X',Y'),Z)$. Assume the joint distribution of $((X,Y),(X',Y'),Z)$ is described by
\begin{align}
    \label{eq: reg}
    &Y = \sum_{j=1}^p m^{(j)}(X^{(j)}) + \epsilon, \quad \epsilon \sim N(0, \sigma^2), \quad m^{(j)}: [0,1] \rightarrow \Real \\
    \label{eq: Z}
    &Z = \frac{1}{2}(Y - Y')^2.
\end{align} 
According to Eqs.~(\ref{eq: reg}, \ref{eq: Z}), the distance is one-half the squared difference of an additive regression response variable. Under this model specification, it is straightforward to show that
\begin{equation}
\begin{aligned}
\label{eq: gb}
 \gb &= \frac{1}{2}(\mb - \mbp)^2 + \sigma^2 \\
 &= \tgb + \sigma^2,
\end{aligned}
\end{equation}
 where $\mb = \Exp[Y|X]$ is the Bayes optimal regression function for predicting $Y$ under squared error loss and $\tgb$ is the Bayes optimal distance predictor if $\epsilon$ in (\ref{eq: reg}) is a constant (see Appendix \ref{sec: bayes} for a detailed derivation). If $Y$ is observable in the training sample and (\ref{eq: reg}, \ref{eq: Z}) are assumed, then one obvious approach for constructing an estimate of $\gb$ would be to use $\tilde{T}_n = \{(X_1,Y_1),\ldots,(X_n,Y_n)\}$ to construct estimates of $\mb$ and $\sigma^2$, and plug them into (\ref{eq: gb}). Denoting by $m_n$ and $s_n = \frac{1}{n}\sum_{i=1}^n(m_n(X) - Y_i)^2$ the estimates of $\mb$ and $\sigma^2$, respectively, consider the estimate
 \begin{align}
 \label{eq: ghat-line1}
 \hat{g}_n(X,X') &= \frac{1}{2}(m_n(X) - m_n(X'))^2 + s_n \\
 \label{eq: ghat}
 &= \tilde{g}_n(X,X') + s_n.
 \end{align}

We can show that a consistent estimate of $\tgb$ exists, using random forests. Scornet et al. \cite{scornet2015} proved $\LL^2$ consistency of regression random forests in the context of the additive regression model described by Eqn.~(\ref{eq: reg}). This result, which we refer to as Scornet Theorem 2 (ST2), is reviewed in Appendix \ref{sec: st2}. We will build off of ST2 to analyze asymptotic performance of distance estimates $g_n(X,X')$ constructed using random forest procedures. First, we define more notation.

A regression forest estimate $m_{n,\Rf}$ is an ensemble of $B$ randomized regression trees constructed on training sample $\tilde{T}_n = \{(X_1,Y_1),\ldots,(X_n,Y_n)\}$, where the randomization procedure is the same as that defined for \Smerf~in Section \ref{sec: 2}. The goal is to minimize $\Exp[(m_{n,\Rf}(X;\tilde{T}_n) - Y)^2]$. Each tree is constructed from a subsample of the original $n$ training points. Denote by $a_n \in \{1,\ldots,n\}$ the specified size of this subsample, and denote by $t_n \in \{1,\ldots,a_n\}$ the specified number of leaves in each tree. A tree is fully grown if $t_n = a_n$, meaning each leaf contains exactly one of the $a_n$ points. The prediction of the response at query point $X$ for the $b^{th}$ tree is denoted by $m_{n,\Rf}(X; \Theta_b, \tilde{T}_n)$, where $\Theta_1,\ldots,\Theta_B$ are i.i.d random variables which are used to randomize each decision tree. The forest estimate is the average of the tree estimates
\begin{align*}
    m_{M,n,\Rf}(X;\Theta_1,\ldots,\Theta_B,\tilde{T}_n) = \frac{1}{B} \sum_{b=1}^B m_{n,\Rf}(X;\Theta_b,\tilde{T}_n)
\end{align*}
To make the analysis more tractable, we take the limit as $B \rightarrow \infty$, obtaining the infinite random forest estimate
\begin{align*}
    m_{n,\Rf}(X;\tilde{T}_n) = \lim\limits_{B \rightarrow \infty} m_{B,n,\Rf}(X;\Theta_1,\ldots,\Theta_B,\tilde{T}_n) = \Exp_\Theta[m_{n,\Rf}(X; \Theta,\tilde{T}_n)]
\end{align*}
Here, expectation is taken with respect to $\Theta$ conditioned on $T_n$ and $X$.
 
\begin{thm}
Suppose the conditions in ST2 are satisfied. Then the estimate $\tilde{g}_{n,\Rf}(X,X') = \frac{1}{2}(m_{n,\Rf}(X) - m_{n\Rf}(X'))^2 \xrightarrow{P} \tgb$
\end{thm}
\begin{proof}
By noting that convergence in mean square implies convergence in probability, the result follows directly from ST2 and Corollary 1 (see Appendix \ref{sec: lemma1-cor1}).
\end{proof}

So far, we have considered estimates $g_n(X,X')$ constructed from regression estimates $m_n(X)$. Such procedures require that $Y$ is accessible in the training sample, which is not always the case.
Now we assume $Y$ is latent, and thus $g_n(X,X')$ can only be constructed by observing each $(X_i,X_j,Z_{ij})$ in the training sample. In this case, we use \Smerf~to construct $g_{n,\Sm}(X,X')$. Under assumption (\ref{eq: Z}), Proposition 2 states that a \Smerf~tree constructed from $T_n = \{(X_i,X_j,Z_{ij}): i = 1,\ldots,n, j \geq i\}$ is identical to a regression tree constructed from $\tilde{T}_n = \{(X_i,Y_i): i = 1,\ldots,n\}$. Thus, we may examine the \Smerf~estimate $g_{n,\Sm}(X,X')$ in terms of the regression tree estimates $m_{n,\Rf}(X)$ and $m_{n,\Rf}(X')$, for which ST2 provides consistency.

\Smerf~ builds $B$ randomized trees constructed using Eq.~(\ref{eq: opt-smerf}). The prediction of the distance at query pair $(X, X')$ made by the $b^{th}$ tree is denoted by $g_{n,\Sm}(X,X';\Theta_b)$, where as before $\Theta_b$ is the randomization parameter for each tree. Now, assume trees are fully grown (meaning one point per leaf). By Proposition 2, the $k^{th}$ \Smerf~tree is equivalent to the $k^{th}$ fully grown regression tree constructed with the same randomization parameter. Having equivalence of leaves between the two trees, denote by $l_b(X_i)$ the leaf of the $b^{th}$ tree containing the single training point $X_i$. Similarly denote by $l_b(X)$ and $l_b(X')$ the leaves that $X$ and $X'$ fall into at prediction time. The $b^{th}$ fully grown regression tree makes the following prediction at query point $X$:
\begin{align}
    m_{n,\Rf}(X;\Theta_b) = \sum_{i=1}^n\II[l_b(X) = l_b(X_i)]Y_i.
\end{align}
That is, the tree makes prediction $Y_i$ when $X$ falls into the same leaf containing the single training point $X_i$. The $b^{th}$ fully grown \Smerf~tree makes the distance prediction at query pair $(X,X')$:
\begin{equation}
\begin{aligned}
    g_{n,\Sm}(X,X';\Theta_b) &= \sum_{i=1}^n\sum_{j=1}^n\II[l_b(X) = l_b(X_i)]\II[l_b(X') = l_b(X_j)]Z_{ij} \\
    &= \sum_{i=1}^n\sum_{j=1}^n\II[l_b(X) = l_b(X_i)]\II[l_b(X') = l_b(X_j)]\frac{1}{2}(Y_i - Y_j)^2.
\end{aligned}
\end{equation}
That is, the tree makes prediction $\frac{1}{2}(Y_i - Y_j)^2$ when $X$ falls into the leaf containing $X_i$ and $X'$ falls into the leaf containing $X_j$. Thus, the following relationship between $m_{n,\Rf}(X;\Theta_b,\tilde{T}_n)$, $m_{n,\Rf}(X';\Theta_b,\tilde{T}_n)$, and $g_{n,\Sm}(X,X';\Theta_b,T_n)$ holds:
\begin{align}
    \label{eq: smerf-rf}
    g_{n,\Sm}(X,X';\Theta_b,T_n) = \frac{1}{2}(m_{n,\Rf}(X;\Theta_b,\tilde{T}_n)) - m_{n,\Rf}(X';\Theta_b,\tilde{T}_n)))^2.
\end{align}

The \Smerf~estimate for $B$ trees is
\begin{align*}
    g_{B,n,\Sm}(X,X';\Theta_1,\ldots,\Theta_B,T_n) = \frac{1}{B}\sum_{b=1}^B g_{n,\Sm}(X,X';\Theta_b,T_n)
\end{align*}
As was done for regression random forests, we analyze the infinite \Smerf~ estimate 
\begin{align}
    \label{eq: smerf-inf}
    g_{n,\Sm}(X,X';T_n) = \Exp_\Theta[g_{n,\Sm}(X,X';\Theta,T_n)]
\end{align}

Substituting (\ref{eq: smerf-rf}) into (\ref{eq: smerf-inf}), we obtain:
\begin{equation}
\begin{aligned}
    \label{eq: gsm}
    g_{n,\Sm}(X,X';T_n) &= \Exp_\Theta[\frac{1}{2}(m_{n,\Rf}(X;\Theta_b,\tilde{T}_n)) - m_{n,\Rf}(X';\Theta_b,\tilde{T}_n)))^2] \\
    &= \frac{1}{2}\Exp_\Theta[m_{n,\Rf}(X;\Theta_b,\tilde{T}_n)) - m_{n,\Rf}(X';\Theta_b,\tilde{T}_n))]^2 \\
    &+ \frac{1}{2}Var_\Theta(m_{n,\Rf}(X;\Theta_b,\tilde{T}_n)) - m_{n,\Rf}(X';\Theta_b,\tilde{T}_n)) \\
    &= \frac{1}{2}(m_{n,\Rf}(X;\tilde{T}_n)) - m_{n,\Rf}(X';\tilde{T}_n)))^2 \\
    &+ \frac{1}{2}Var_\Theta(m_{n,\Rf}(X;\Theta_b,\tilde{T}_n)) - m_{n,\Rf}(X';\Theta_b,\tilde{T}_n)) \\
    &= \tilde{g}_{n,\Rf}(X,X') + s_{n,\Rf} \\
\end{aligned}
\end{equation}
Comparing the last line of (\ref{eq: gsm}) to (\ref{eq: ghat}), we see that $g_{n,\Sm}(X,X')$ is implicitly an estimate of the form $\hat{g}_n(X,X')$. Thus, \Smerf~estimates two contributions to $\gb$: The $\tilde{g}_{n,\Rf}$ term estimates one-half the squared deviation of $\Exp[Y|X]$ and $\Exp[Y'|X']$, while the $s_{n,\Rf}$ term estimates $\sigma^2$. By Theorem 1, we know that $\tilde{g}_{n,\Rf}(X,X') \xrightarrow{P} \tgb$. Thus, Lemma 1 (Appendix \ref{sec: lemma1-cor1}) tells us half of the work is done  in establishing consistency of $g_{n,\Sm}(X,X')$; the second half is establishing that $s_{n,\Rf} \xrightarrow{P} \sigma^2$. Unfortunately, this term is difficult to analyze due to the fact that $m_{n,\Rf}(X;\Theta_b,\tilde{T}_n)$ and $m_{n,\Rf}(X';\Theta_b,\tilde{T}_n)$ are not independent with respect to the distribution of $\Theta$. However, our own numerical experiments for various additive regression settings suggest that $s_{n,\Rf}$ gets arbitrarily close to $\sigma^2$ with large $n$ (Appendix \ref{sec: empirical-convergence}). Based on these empirical findings, we conjecture that $\Smerf$ is consistent under our framework. Confirmation of this is left for future work.

\section{Experiments}
\label{sec: experiments}

\subsection{Simulations for Distance Learning}
\label{sec: experiments-simulations}
We evaluate the ability of \Smerf~to learn distances in three very different simulated settings. \Smerf's performance is compared to two other methods which are designed to learn quite different notions of distances. The first method we compare to learns a symmetric squared Mahalanobis distance by solving the following optimization problem:
\begin{align*}
\min_\Wm \sum_{i,j} ((\mathbf{x}_i - \mathbf{x}_j)^T \Wm (\mathbf{x}_i - \mathbf{x}_j) - \zij)^2,\quad s.t.\, \Wm \succcurlyeq 0
\end{align*}
where $\mathbf{x}$ and $\zij$ are the same as in Section \ref{sec: 2}. This can be seen as a regression form of \cite{xing2003}. We refer to this method as \Ma. The second method we compare to learns a bilinear similarity \cite{chechik2010,oswal2016,kulis2012} via the optimization:
\begin{align*}
\min_\Wm ||\Xm^T \Wm \Xm - \Qm||_F^2,\quad s.t.\, \Wm \succcurlyeq 0,
\end{align*}
where $\Xm \in \Real^{n \times p}$ is the matrix form of $\{\mathbf{x}_i\}_1^n$ and $\Qm \in \Real^{n \times n}$ is a symmetric matrix whose element $q_{ij}$ is the similarity between $\mathbf{x}_i$ and $\mathbf{x}_j$. The learned matrix $\Wm$ can be viewed as a linear mapping to a new inner product space, such that the dot product between two points $\mathbf{x}_i$ and $\mathbf{x}_j$ after mapping to the new space is (hopefully) close to $q_{ij}$. We refer to this method as $\Bi$.

We implement $\Ma$ and $\Bi$ in Matlab using the CVX package with the Mosek solver, while $\Smerf$ was implemented in native R. In all three experiments, the number of training examples ranged from 20 to 320 and the number of test examples was 200. The dimensionality of the input space, $p$, for all three experiments is 20, but only the first two dimensions contribute to the distance and the other dimensions are irrelevant. Each experiment was repeated ten times. The experiments are described below, with additional details in Appendix \ref{sec: simulated-datasets}.

\textbf{Regression Distance} models the pairwise distances as the squared deviation between additive regression responses. Specifically, each dimension of each $\mathbf{x}_i$ is sampled i.i.d. from $U(0.1, 0.9)$. Then regression responses $\{y_i\}_1^n$ are computed according to
\begin{align}
    \label{eq: addreg}
    y_i = \frac{1}{2} (x_i^{(1)} + x_i^{(2)}) + \epsilon,
\end{align}
where $x_i^{(j)}$ is the $j^{th}$ dimension of $\mathbf{x}_i$ and $\epsilon \sim U(-0.1,0.1)$. The pairwise distance is defined as $\zij = (y_i - y_j)^2$, which is bounded in $[0,1]$. This boundedness was specified purposefully since \Bi~requires a similarity matrix $\Qm$ for training and testing. Thus, we have a natural conversion from distance to similarity using the formula $\Qm = \mathbf{1} - \Zm$.

\textbf{Bilinear Distance} models the similarities between two points as the product of their regression responses. Specifically, each dimension of each $\mathbf{x}_i$ is sampled i.i.d. from $U(0, 1)$. Regression responses $\{y_i\}_1^n$ are derived from (\ref{eq: addreg}), but with $\epsilon$ removed. Then similarity $q_{ij} = y_i y_j$, which is bounded in $[0,1]$. We obtain $\Zm = \mathbf{1} - \Qm$ as input to \Smerf~and \Ma.

\textbf{Radial Distance} models the distance between two points contained to the unit ball by the squared deviation of the vector norms of their first two dimensions. Specifically each $\mathbf{x}_i$ is uniformly distributed within the 20-dimensional unit ball. Letting $\mathbf{x}^{(1:2)}$ denote the first two dimensions of $\mathbf{x}$, the distance is $\zij = (|\mathbf{x}_i^{(1:2)}| - |\mathbf{x}_j^{(1:2)}|)^2$, which is bounded in $[0,1]$. Again, we obtain $\Qm$ as $\mathbf{1} - \Zm$.

\begin{figure}[t]
\centering
\centerline{\includegraphics[width=\columnwidth,trim={0in 0in 0 0},clip]{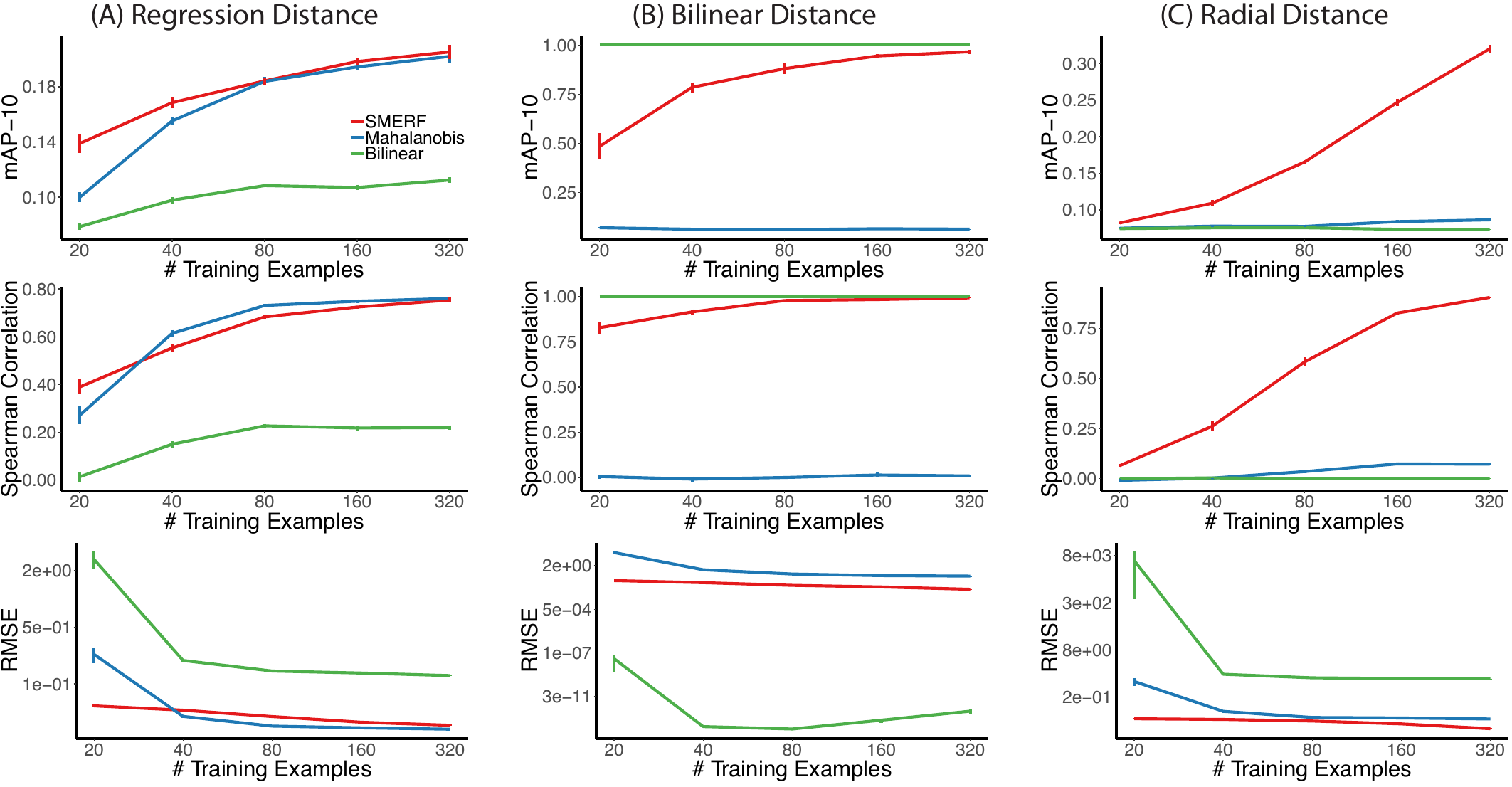}}
\caption{The mAP-10 (top row), Spearman Correlation (middle row), and RMSE (bottom row) performance measures on the three simulated experiments. Error bars represent SEM. \Smerf~is the only method that can learn all three distances reasonably well.}
\label{fig: sim}
\end{figure}

Figure \ref{fig: sim} shows three performance measures: 1) mAP-10 (top row) is the mean average precision, using the ten ground-truth closest points to each test point as the relevant items; 2) Spearman (middle row) is the average Spearman correlation between the predicted and ground-truth distances between each test point and every other point; 3) RMSE (bottom row) is the root-mean-squared-error between the predicted and ground-truth distances.

The Regression Distance experiment was designed specifically for \Ma~to perform well. The left column of Figure \ref{fig: sim} shows that \Bi~performs poorly in all three measures, while \Smerf~performs comparably to \Ma, perhaps even better for small sample sizes. The Bilinear Distance experiment was designed specifically for \Bi~to perform well. Similarly, we see in the middle column that \Ma~completely fails in all three measures while \Smerf~eventually performs comparably to \Bi. The right column shows that \Smerf~is the only method that can learn the Radial Distance. Furthermore, \Smerf~is able to correctly identify the first two dimensions as the important dimensions for the Radial Distance (Appendix \ref{sec: feature-importance} and Figure \ref{fig: feature-importance}). Overall, these results highlight the robustness and interpretability of \Smerf.

\subsection{Network Link Prediction}
\label{sec: experiments-networks}
Related to the notion of distances/similarities between items is the notion of interactions between items. Here, we demonstrate the flexibility and real-world utility of \Smerf~by using it to predict links in a network when node attribute information is available. We compare our method to two state-of-the-art methods for predicting links in a network. One is the Edge Partition Model (\Epm) \cite{zhou2015}, which is a Bayesian latent factor relational model. It is purely relational, meaning it does not account for node attributes. The second is the Node Attribute Relational Model (\Narm) \cite{zhao2017}, which builds off of \Epm~by incorporating node attribute information into the model priors. Thus, the methods span a gradient from only using network structural information (\Epm) to only using node attribute information (\Smerf). We compare the three methods on three real-world network data sets used in \cite{zhao2017}. \textbf{Lazega-cowork} is a network of cowork relationships among 71 attorneys, containing 378 links. Each attorney is associated with eight ordinal and binary attributes. \textbf{Facebook-ego} is a network of 347 Facebook users with 2419 links. Each user is associated with 227 binary attributes. \textbf{NIPS234} is a network of 234 NIPS authors with 598 links indicating co-authorship. Each author is associated with 100 binary attributes. Additional details regarding the network data sets and experimental setup can be found in Appendix \ref{sec: network-supplementary}.

For each data set, the proportion of nodes used for training (TP) was varied from 0.1 to 0.9 (the rest used for testing). We note that this is different from \cite{zhao2017}, in which the data was split by node-pairs, rather than by nodes. For EPM, however, we split the data by node-pairs. This is because, by not leveraging node attributes, it is hopeless in being able to predict links for newly observed nodes any better than chance. Experiments for each TP were repeated five times. For \Smerf, $\Zm$ was computed as $\mathbf{1} - \mathbf{A}$, where $\mathbf{A}$ is the adjacency matrix. Thus, \Smerf~predictions represent scores between zero and one reflecting the belief that a link exists. \Epm~and \Narm~explicitly model such beliefs. Thus, we use the area under the ROC (AUC-ROC) and Precision-Recall (AUC-PR) curves as measures of performance.

\begin{figure}[t]
\centering
\centerline{\includegraphics[width=\columnwidth,trim={0in 0in 0 0},clip]{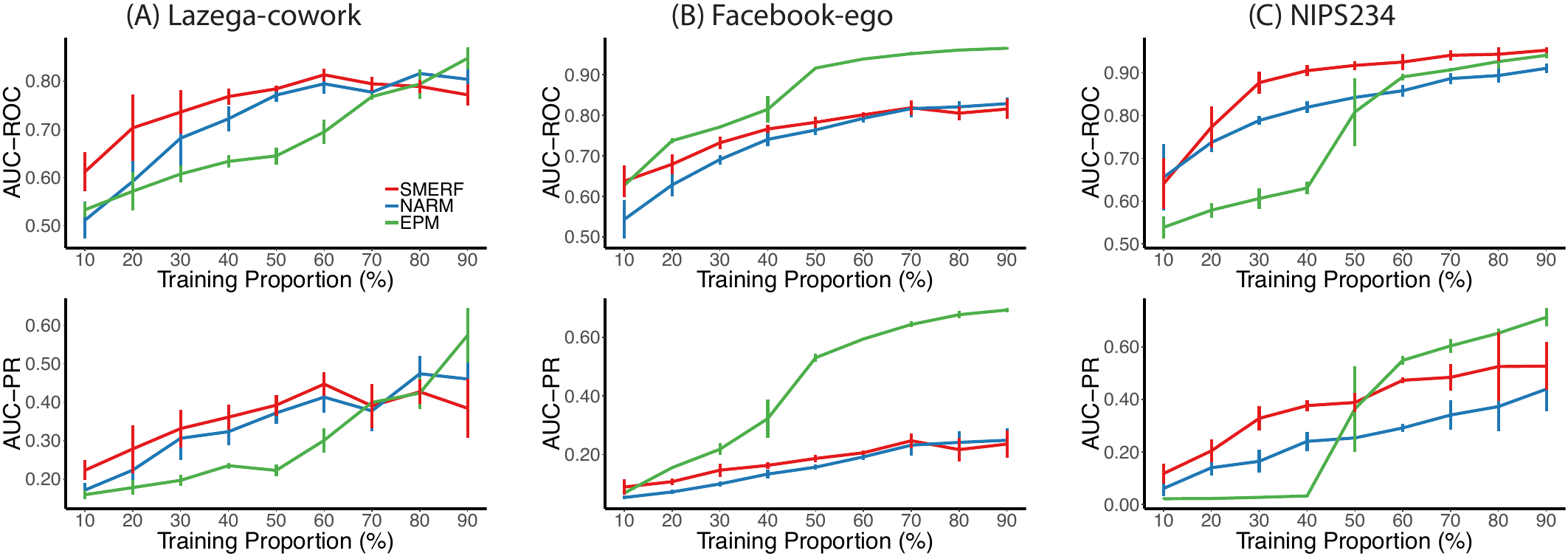}}
\caption{The AUC of the ROC (top row) and precision-recall curve (bottom row) for the three network data sets, as a function of the proportion of data used for training. Error bars represent the standard deviation over five replicate experiments. \Smerf~is competitive with state-of-the-art relational learning methods, performing the best in many settings.}
\label{fig: net}
\end{figure}

Figure \ref{fig: net} shows AUC-ROC (top) and AUC-PR (bottom) for the three networks. The left column indicates shows that for small TP on Lazega-cowork, \Smerf~outperforms \Narm~for AUC-ROC, while \Narm~eventually catches up. They perform comparably in terms of AUC-PR. Both methods outperform \Epm~until the TP equals 70\%. On Facebook-ego, \Epm~outperforms the other methods. \Smerf~performs slight better than \Narm~for small TP and comparably otherwise. On NIPS234, \Smerf~outperforms \Narm~in both measures for nearly all values of TP. \Smerf~substantially outperforms \Epm~for small TP, but \Epm~quickly begins to win once TP equals 50\%. Overall, the results indicate that the general purpose \Smerf, which disregards network structural information, is highly competitive with dedicated relational models for the purpose of link prediction.

\section{Conclusion}
We have presented a novel tree ensemble-based method called \Smerf~for learning distances, which generalizes classification and regression trees. Via its connection to regression forests, analysis of \Smerf~under a statistical learning framework is performed through the lens of regression forest estimates. We show that \Smerf~can robustly learn many notions of distance, and also demonstrate how it can be used for network link prediction. Future work will build off of theoretical work here for establishing consistency of \Smerf. Second, we will explore how different notions of distance constructed from labeled data can impact k-NN accuracy using the learned \Smerf~distance function. Third, \Smerf~will be extended to handle missing values in the distance matrix during training. Last, we will investigate ways to improve performance of \Smerf~for link prediction by leveraging the addition of network structural information.

\section{Acknowledgements}
Research was partially supported by funding from Microsoft Research.





\bibliography{smerf}
\bibliographystyle{unsrtnat}

\clearpage

\begin{appendices}

\input{supplementary_material_nd}

\end{appendices}

\end{document}

%% file: nd_preamble.tex
\usepackage[utf8]{inputenc} 
\usepackage[T1]{fontenc}    
\usepackage{hyperref}       
\usepackage{url}            
\usepackage{booktabs}       
\usepackage{nicefrac}       
\usepackage{microtype}      
\usepackage{graphicx}
\usepackage{mathtools}
\usepackage{atbegshi}
\usepackage{array}
\usepackage{longtable}
\usepackage{pdflscape}
\usepackage{color}
\usepackage{multirow}
\usepackage{tablefootnote}
\usepackage[title]{appendix}

\usepackage{lipsum}

\usepackage{endnotes}

\usepackage{titlesec}
\usepackage{titletoc}

\titleclass{\task}{straight}[\section]
\newcounter{task}
\renewcommand{\thetask}{\arabic{task}}
\titleformat{\task}[hang]
    {\normalfont\LARGE\bfseries}{Task \thetask:}{1em}{}

\titleformat*{\task}{\color{header1}\bfseries}

\titlecontents{task}
              [3.8em] 
              {}
              {\contentslabel{2.3em}}
              {\hspace*{-2.3em}}
              {\titlerule*[1pc]{.}\contentspage}

\titlespacing*{\section}{0ex}{1ex}{1ex}
\titlespacing*{\subsection}{0ex}{1ex}{1ex}
\titlespacing*{\paragraph}{0ex}{1ex}{1ex}
\titlespacing*{\subparagraph}{0ex}{1ex}{1ex}
\titlespacing*{\task}{0em}{1ex}{1ex}

\usepackage[sort&compress,comma,square,numbers]{natbib}

\usepackage{amsfonts,amsmath,amssymb}
\usepackage{bbm}

\usepackage{multicol}
\usepackage{enumitem}

\usepackage{hhline}

\usepackage{comment}
\usepackage{todonotes}
\RequirePackage{ulem}

\usepackage{booktabs}
\usepackage{titlesec}
\usepackage{fancyhdr}
\usepackage{fullpage}

\RequirePackage{titlesec}
\RequirePackage[font={small},{singlelinecheck=false}]{caption}
\usepackage[T1]{fontenc}
\usepackage[scaled]{helvet}
\usepackage[scaled=1.10, med]{zlmtt}

\usepackage{algorithm}
\usepackage{algpseudocode}

\newcommand{\argmax}{\operatornamewithlimits{argmax}}

\newcommand{\Real}{\mathbb{R}}
\newcommand{\Exp}{\mathbb{E}}
\newcommand{\II}{\mathbb{I}}
\newcommand{\LL}{\mathbb{L}}
\newcommand{\Xc}{\mathcal{X}}
\newcommand{\Yc}{\mathcal{Y}}
\newcommand{\Cc}{\mathcal{C}}
\newcommand{\Sc}{\mathcal{S}}
\newcommand{\Ig}{I_\mathcal{G}}
\newcommand{\Id}{I_\mathcal{D}}
\newcommand{\Iv}{I_\mathcal{V}}
\newcommand{\zij}{z_{ij}}
\newcommand{\nsk}{n_s^{(k)}}
\newcommand{\nsl}{n_s^{(l)}}
\newcommand{\ns}{n_s}
\newcommand{\gb}{g^*(X,X')}
\newcommand{\tgb}{\tilde{g}^*(X,X')}
\newcommand{\mb}{m^*(X)}
\newcommand{\mbp}{m^*(X')}
\newcommand{\Zm}{\mathbf{Z}}
\newcommand{\Xm}{\mathbf{X}}
\newcommand{\Qm}{\mathbf{Q}}
\newcommand{\Wm}{\mathbf{W}}
\newcommand{\Am}{\mathbf{A}}
\newcommand{\xvi}{\mathbf{x}_i}
\newcommand{\xvj}{\mathbf{x}_j}
\newcommand{\woo}{w_{11}}
\newcommand{\wot}{w_{12}}
\newcommand{\wto}{w_{21}}
\newcommand{\wtt}{w_{22}}

\newcommand{\Smerf}{{\sc \texttt{SMERF}}}
\newcommand{\Sm}{{\sc \texttt{SM}}}
\newcommand{\Rf}{{\sc \texttt{RF}}}
\newcommand{\Ma}{{\sc \texttt{Mahalanobis}}}
\newcommand{\Bi}{{\sc \texttt{Bilinear}}}
\newcommand{\Epm}{{\sc \texttt{EPM}}}
\newcommand{\Narm}{{\sc \texttt{NARM}}}
\newcommand{\dd}{{ \texttt{d}}}
\newcommand{\mm}{{ \texttt{min.parent}}}
\newcommand{\ff}{{ \texttt{FUN}}}
\newcommand{\Rmb}{{ \texttt{RandMatBinary}}}
\newcommand{\Rmr}{{ \texttt{RandMatRF}}}

\newcommand\defeq{\stackrel{\mathclap{\normalfont\mbox{def}}}{=}}

\newtheorem{thm}{Theorem}
\newtheorem{prop}{Proposition}
\newtheorem{lem}{Lemma}
\newtheorem{cor}{Corollary}
\newtheorem*{scor}{Scornet et al. Theorem 2}

%% file: supplementary_material_nd.tex
\section{Proofs of Propositions 1 and 2}
\subsection{Proof of Proposition 1}
Consider two trees partitioning a set of points $\{\mathbf{x}\}_1^n$, constructed with the same stopping criteria. Without loss of generality, assume the trees are deterministic. Then the trees will be identical if the split optimizations are identical. Thus, it suffices to show that if $\zij = \II[c_i \neq c_j]$, then $\Id$ in (\ref{eq: opt-smerf}) is identical to $\Ig$ in (\ref{eq: opt-class}).

First, compute $\Id(\Sc)$ for a sample of points $\Sc$. The sum of all pairwise distances $\zij$ is just the total number of pairs of points in $\Sc$ that don't share the same class label. Denote by $\ns^{(k)}$ the number of points in $\Sc$ whose class label is $k$. For any pair of distinct class labels $k,l \in \{1,\ldots,K\}$, there are $\nsk \nsl$ pairs of points whose members have labels that are either $k$ or $l$, but who do not share the same label. Thus, the total number of pairs of points in $\Sc$ not sharing the same label is $\sum_{k=1}^{K-1} \sum_{l=k+1}^K \nsk \nsl$. Since the distance matrix $\Zm_{\Sc}$ for the set $\Sc$ double counts each pair of distinct points in $\Sc$ (because it is symmetric), we multiply this sum by $2$. Noting that there are $\ns^2$ total pairs of points, the average pairwise distance is
\begin{align*}
    \Id(\Sc) = \frac{2}{\ns^2} \sum_{k=1}^{K-1} \sum_{l=k+1}^K \nsk\nsl
\end{align*}
We show that this expression is equivalent to $\Ig$:
\begin{align*}
    \Ig(\Sc) &= 1 - \sum_{k=1}^K f_k^2 \\
    &= 1 - \sum_{k=1}^K (\frac{\ns^{(k)}}{\ns})^2 \\
    &= \frac{1}{\ns^2}(\ns^2 - \sum_{k=1}^K \ns^{(k)2}) \\
    &= \frac{1}{\ns^2}((\sum_{k=1}^K \nsk)^2 - \sum_{k=1}^K \ns^{(k)2}) \\
    &= \frac{1}{\ns^2}(\sum_{k=1}^K \ns^{(k)2} + 2\sum_{k=1}^{K-1} \sum_{l=k+1}^K \nsk\nsl - \sum_{k=1}^K \ns^{(k)2}) \\
    &= \frac{2}{\ns} \sum_{k=1}^{K-1} \sum_{l=k+1}^K \nsk\nsl \\
    &= \Id(\Sc)
\end{align*}
where in the third to last equality we used the identity $(\sum_{k=1}^K \nsk)^2 = \sum_{k=1}^K \ns^{(k)2} + 2\sum_{k=1}^{K-1} \sum_{l=k+1}^K \nsk\nsl$.
\newpage
\subsection{Proof of Proposition 2}

As in the proof for Proposition 1, it suffices to show that if $\zij = \frac{1}{2}(y_i - y_j)^2$, then $\Id$ in (\ref{eq: opt-smerf}) is identical to $\Iv$ in (\ref{eq: opt-reg}).
\begin{align*}
    \Id(\Sc) &= \frac{1}{\ns^2} \sum_{i,j \in \Sc} \zij\\
             &= \frac{1}{\ns^2} \sum_{i,j \in \Sc} \frac{1}{2}(y_i - y_j)^2\\
             &= \frac{1}{\ns^2} \sum_{\substack{i,j \in \Sc\\
                                   i \neq j}} \frac{1}{2}(y_i - y_j)^2\\
             &= \frac{1}{\ns^2} \sum_{\substack{i,j \in \Sc\\
                                   i \neq j}} \frac{1}{2}(y_i^2 - 2y_iy_j + y_j^2)\\
             &= \frac{1}{\ns^2} (\sum_{\substack{i,j \in \Sc\\
                                   i \neq j}} \frac{1}{2}(y_i^2 + y_j^2) - \sum_{\substack{i,j \in \Sc\\
                                   i \neq j}} y_iy_j)\\
              &= \frac{1}{\ns^2} ((\ns-1)\sum_{i \in \Sc} y_i^2 - \sum_{\substack{i,j \in                       \Sc\\
                                   i \neq j}} y_iy_j)\\
              &= \frac{1}{\ns} \sum_{i \in \Sc} y_i^2 - \frac{1}{\ns^2}(\sum_{i \in \Sc} y_i^2 + \sum_{\substack{i,j \in                       \Sc\\
                                   i \neq j}} y_iy_j)\\
              &= \frac{1}{\ns} \sum_{i \in \Sc} y_i^2 - \frac{1}{\ns^2} (\sum_{i \in \Sc} y_i^2 + (\sum_{i \in \Sc} y_i)^2 - \sum_{i \in \Sc} y_i^2)\\
              &= \frac{1}{\ns} \sum_{i \in \Sc} y_i^2 - (\frac{1}{\ns}\sum_{i \in \Sc} y_i)^2\\
              &= \frac{1}{\ns} \sum_{i \in \Sc} y_i^2 - \bar{y}^2\\
              &= \frac{1}{\ns} \sum_{i \in \Sc} y_i^2 - 2\bar{y}^2 + \bar{y}^2\\
              &= \frac{1}{\ns} \sum_{i \in \Sc} (y_i^2 - 2y_i\bar{y} + \bar{y}^2)\\
              &= \frac{1}{\ns} \sum_{i \in \Sc} (y_i - \bar{y})^2\\
              &= \Iv.
\end{align*}
Going from the fourth to the fifth equality we used the identity
\begin{align*}
    \sum_{\substack{i,j \in \Sc\\
                    i \neq j}} (y_i^2 + y_j^2) = 2(\ns-1)\sum_{i \in \Sc} y_i^2,
\end{align*}
and going from the sixth to the seventh equality we used the identity
\begin{align*}
    \sum_{\substack{i,j \in \Sc\\
                    i \neq j}} y_iy_j = (\sum_{i \in \Sc} y_i)^2 - \sum_{i \in \Sc} y_i^2.
\end{align*}

\newpage

\section{Supplementary Results for Section \ref{sec: theory}}
\subsection{Lemma 1 and Corollary 1}
\label{sec: lemma1-cor1}
Recall the Bayes distance predictor
\begin{align*}
    \gb &= \frac{1}{2}(\mb - m^*(X'))^2 + \sigma^2 \\
    &= \tgb + \sigma^2.
\end{align*}
Further recall the estimate defined by (\ref{eq: ghat-line1}):
\begin{align*}
    \hat{g}_n(X,X') &= \frac{1}{2}(m_n(X) - m_n(X'))^2 + s_n \\
    &= \tilde{g}_n(X,X') + s_n.
\end{align*}

We have the following lemma.

\begin{lem}
    Suppose $m_n(X) \xrightarrow{P} \mb$ and $s_n \xrightarrow{P} \sigma^2$. Then $\hat{g}_n(X,X') \xrightarrow{P} \gb$.
\end{lem}
\begin{proof}
The function $f(a,b) = \frac{1}{2}(a - b)^2$ is continuous in $a$ and $b$. Therefore, it follows from the continuous mapping theorem (CMT) that $\tilde{g}_n(X,X') = f(m_n(X), m_n(X')) \xrightarrow{P} f(\mb, m^*(X')) = \tgb$. Similarly, the function $h(a, b) = a + b$ is continuous in $a$ and $b$. Using the result of the first application of CMT, another application of CMT yields $\hat{g}_n(X,X') = h(\tilde{g}_n(X,X'), s_n) \xrightarrow{P} h(\tgb, \sigma^2) = \gb$. This completes the proof.
\end{proof}

The proof of Lemma 1 leads to the following corollary.

\begin{cor}
Let $m_n(X) \xrightarrow{P} \mb$. Then $\tilde{g}_n(X,X') \xrightarrow{P} \tgb$.
\end{cor}

\subsection{Review of Scornet et al. \cite{scornet2015} Theorem 2 (ST2)}
\label{sec: st2}
\begin{scor}
Assume Equation (\ref{eq: reg}) holds. Let $t_n = a_n$. Then, provided  $a_n \rightarrow \infty$, $t_n \rightarrow \infty$, and $a_nlogn/n \rightarrow 0$, random forests are $\LL^2$ consistent. That is, $\lim\limits_{n \rightarrow \infty} \Exp[(m_{n,\Rf}(X) - \mb)^2] = 0$.
\end{scor}

There is actually another condition which must be satisfied in order to achieve this consistency result. However, it is quite technical and difficult to check in practice, as \cite{scornet2015} states. Therefore, we omit it so not to distract the reader. \cite{scornet2015} notes that consistency holds even if the Gaussian noise term in (\ref{eq: reg}) is replaced with any bounded random variable or constant.

\subsection{Derivation of the Bayes Optimal Distance Predictor $\gb$}
\label{sec: bayes}
We explicitly derive the Bayes distance function $\gb$, which minimizes $L(g) = \Exp[(g(X,X') - Z)^2]$ under assumptions (\ref{eq: reg},\ref{eq: Z}). From the law of total expectation, it suffices to minimize $L(g|X,X') = \Exp[(g(X,X') - Z)^2 | X,X']$ for each $X,X'$ with positive measure. Set the derivative of $L(g|X,X')$ to zero:
\begin{align*}
    0 &= \frac{d}{dg} \Exp[(g(X,X') - Z)^2 | X,X'] \\
    &= \frac{d}{dg} (\Exp[g(X,X')^2 | X,X'] - 2\Exp[g(X,X')Z | X,X'] + \Exp[Z^2 | X,X']) \\
    &= \frac{d}{dg} (g(X,X')^2 - 2g(X,X')\Exp[Z | X,X'] + \Exp[Z^2 | X,X']) \\
    &= 2\gb - 2\Exp[Z | X,X']
\end{align*}
Thus, $\gb = \Exp[Z | X,X']$. Since $\frac{d^2}{dg^2} \Exp[(g(X,X') - Z)^2 | X,X'] = 2 > 0$, $L(g|X,X')$ is strictly convex, and therefore $\Exp[Z | X,X']$ is the global minimizer.

Under assumptions (\ref{eq: reg},\ref{eq: Z}), $\gb$ becomes
\begin{align*}
    \gb &= \Exp[\frac{1}{2}(Y - Y')^2 | X,X'] \\
    &= \frac{1}{2}(\Exp[Y^2 | X,X'] - 2\Exp[YY' | X,X'] + \Exp[Y'^2 | X,X']) \\
    &= \frac{1}{2}(\Exp[Y | X]^2 + Var(Y | X) - 2\Exp[Y | X]\Exp[Y' | X'] + \Exp[Y' | X']^2 + Var(Y' | X')) \\
    &= \frac{1}{2}(\Exp[Y | X] - \Exp[Y' | X'])^2 + \frac{1}{2}(Var(Y | X) + Var(Y' | X')) \\
    &= \frac{1}{2}(\mb - m^*(X'))^2 + \frac{1}{2}(\sigma^2 + \sigma^2) \\
    &= \frac{1}{2}(\mb - m^*(X'))^2 + \sigma^2
\end{align*}

\subsection{Empirical convergence of $s_{n,\Rf}$ to $\sigma^2$}
\label{sec: empirical-convergence}
We noted in Section \ref{sec: theory} that $s_{,\Rf}$ empirically converges to $\sigma^2$. We simulate data from an additive regression model of the form:
\begin{align*}
    Y = X^{(1)2} + X^{(2)2} + \epsilon,
\end{align*}
where $X^{(1)}$ and $X^{(2)}$ are distributed uniformly in $[0,1]$ and $\epsilon \sim N(0, 0.01)$. We train a regression random forest on a random sample of $n$ training points, for $n$ ranging from $2^4$ to $2^16$. We make predictions on a separate random sample of $200$ test points. This is repeated 10 times for each value of $n$. 

We cannot exactly compute the infinite forest estimate $s_{n,\Rf}(\mathbf{x},\mathbf{x'}) = Var_\Theta(m_n(\mathbf{x}) - m_n(\mathbf{x}'))$ for a pair of points, but we can reasonably approximate with many trees. Therefore, we used $1000$ trees, which we deemed sufficient. We compute $s_{n,\Rf}(\mathbf{x},\mathbf{x'})$ for all pairs of points $\mathbf{x},\mathbf{x'}$ in the test set, then take the average over all pairs of test points as the empirical estimate $s_{n,\Rf}$. Figure \ref{fig: sn} shows that $s_{n,\Rf}$ asymptotically approaches $\sigma^2 = 0.01$.

\begin{figure}
\centering
\centerline{\includegraphics[width=0.5\columnwidth,trim={0in 0in 0 0},clip]{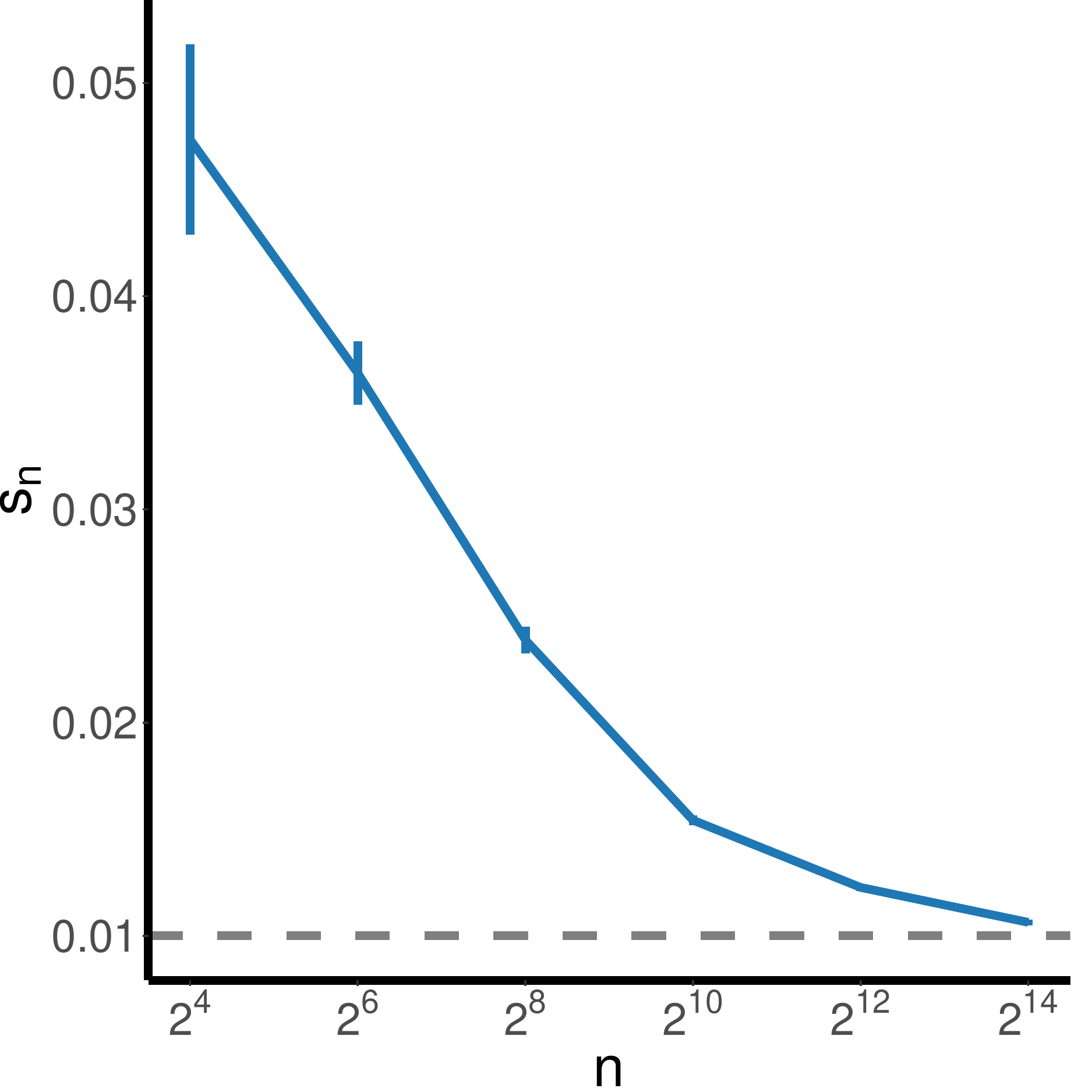}}
\caption{$s_{n,\Rf}$ empirically converges to $\sigma^2 = 0.01$.}
\label{fig: sn}
\end{figure}

\section{\Smerf~Implementation Details}
\Smerf~was built on top of the Sparse Projection Oblique Randomer Forest (SPORF) R package (\url{https://github.com/neurodata/SPORF}). SPORF is an extension of vanilla random forests that allows for arbitrarily-oriented splits, whereas vanilla random forests only splits orthogonal to the input feature axes. SPORF does so by randomly projecting the input from the original space in $\Real^p$ to a new space in $\Real^d$, where $d$ can be greater than $p$. Then, splits are searched along each dimension in the projected space.

One of the options in SPORF is \ff, which specifies whether SPORF should randomly project the data or just randomly subsample input feature dimensions (i.e. vanilla random forest). Performing random projections is done by setting $\ff = \Rmb$, while performing vanilla random forests is done by setting $\ff = \Rmr$.

For the experiments in Section \ref{sec: experiments-simulations}, we set $\ff = \Rmb$. For the experiments in Section \ref{sec: experiments-networks}, we set $\ff = \Rmr$.

\section{Supplementary Information on Experiments in Section \ref{sec: experiments-simulations}}

\subsection{\Ma~and \Bi~Algorithms}
\Ma~learns a distance of the form
\begin{align*}
    \zij = (\xvi - \xvj)^T \Wm (\xvi - \xvj),\quad \Wm \succcurlyeq 0.
\end{align*}
This defines a class of \emph{squared} Mahalanobis distances. Since $\Wm$ is positive semidefinite, it can be factorized as $\Wm = \Am^T \Am$, where $\Am = \Wm^{1/2}$. Therefore, the squared Mahalanobis distance can be expressed as
\begin{align*}
    \zij &= (\xvi - \xvj)^T \Am^T \Am (\xvi - \xvj)\\
         &\defeq (\tilde{\mathbf{x}}_i - \tilde{\mathbf{x}}_j)^T (\tilde{\mathbf{x}}_i - \tilde{\mathbf{x}}_j)\\
\end{align*}
It then follows that learning a (squared) Mahalanobis distance parameterized by the class of symmetric positive semidefinite $\Wm$ is equivalent to learning a linear transformation of the input and applying the (squared) Euclidean distance to the transformed input.

Similarly, \Bi~ learns a similarity of the form
\begin{align*}
    q_{ij} = \xvi^T \Wm \xvj,\quad \Wm \succcurlyeq 0
\end{align*}
This parameterizes a class of bilinear similarities. Again, since $\Wm$ is positive semidefinite, $q_{ij}$ can be expressed as
\begin{align*}
    q_{ij} &= \xvi^T \Am^T \Am \xvj\\
         &\defeq \tilde{\mathbf{x}}_i^T \tilde{\mathbf{x}}_j
\end{align*}
It then follows that learning a bilinear similarity parameterized by the class of symmetric positive semidefinite $\Wm$ is equivalent to learning a linear transformation of the input and applying the inner product to the transformed input.

\subsection{Simulated Data Sets}
\label{sec: simulated-datasets}
\subsubsection{Regression Distance}
We claimed in section \ref{sec: experiments-simulations} that the Regression Distance simulated data set was designed specifically for \Ma~ to perform well. Here we show that the distance $\zij$ defined for this data set is precisely a squared Mahalanobis distance.

The squared Mahalanobis distance between $\mathbf{x}_i$ and $\mathbf{x}_i$ in $\Real^2$ takes the form
\begin{align*}
    z_{ij,\Ma} &= (\mathbf{x}_i - \mathbf{x}_j)^T \Wm (\mathbf{x}_i - \mathbf{x}_j)\\
                 &= \woo(x_i^{(1)} - x_j^{(1)})^2 + \wto(x_i^{(2)} - x_j^{2)})(x_i^{(1)}- x_j^{(1)})\\
                 &+ \wot(x_i^{(1)} - x_j^{(1)})(x_i^{(2)} - x_j^{(2)}) + \wtt(x_i^{(2)} - x_j^{(2)})^2\\
                 &= \woo(x_i^{(1)} - x_j^{(1)})^2 + 2\wot(x_i^{(1)} - x_j^{(1)})(x_i^{(2)} - x_j^{(2)}) + \wtt(x_i^{(2)} - x_j^{(2)})^2
\end{align*}
where $w_{ij}$ is the $(i,j)^{th}$ element of $\Wm$ and $w_{ij} = w_{ji}$.

Now, consider the noiseless version of (\ref{eq: addreg})
\begin{align}
    \label{eq: addreg-noiseless}
    y_i = \frac{1}{2}(x_i^{(1)} + x_i^{(2)}),
\end{align}
and recall that we defined the pairwise distance for the Regression Distance data set to be $z_{ij,RegDist} = (y_i - y_j)^2$. Then
\begin{align*}
    z_{ij,RegDist} &= (y_i - y_j)^2\\
               &= (\frac{1}{2}(x_i^{(1)} + x_i^{(2)}) - \frac{1}{2}(x_j^{(1)} + x_j^{(2)}))^2\\
               &= \frac{1}{4}((x_i^{(1)} - x_j^{(1)}) + (x_i^{(2)} - x_j^{(2)}))^2\\
               &= \frac{1}{4}((x_i^{(1)} - x_j^{(1)})^2 + \frac{1}{2}(x_i^{(1)} - x_j^{(1)})(x_i^{(2)} - x_j^{(2)}) + \frac{1}{4}(x_i^{(2)} - x_j^{(2)})^2
\end{align*}
Comparing $z_{ij,RegDist}$ to $z_{ij,\Ma}$, we see that $z_{ij,RegDist}$ is a squared Mahalanobis distance with $\woo = \wot = \wtt = \frac{1}{4}$.

\subsubsection{Bilinear Distance}
Similarly, we claimed that the aptly named Bilinear Distance simulated data set was designed specifically for \Bi~ to perform well. Here we show that the similarity $q_{ij}$ (for which the distance is equal to 1 minus the similarity) defined for this data set is, in fact, a bilinear similarity.

The bilinear similarity between $\mathbf{x}_i$ and $\mathbf{x}_j$ in $\Real^2$ takes the form
\begin{align*}
    q_{ij,\Bi} &= \mathbf{x}_i^T \Wm \mathbf{x}_j\\
                    &= \woo x_i^{(1)}x_j^{(1)} + \wto x_i^{(2)}x_j^{(1)} + \wot x_i^{(1)} x_j^{(2)} + \wtt x_i^{(2)}x_j^{(2)}.
\end{align*}

Recall that $y_i$ is defined as in (\ref{eq: addreg-noiseless}), and that we defined the pairwise similarity for the Bilinear Distance data set to be
\begin{align*}
    q_{ij,BiDist} &= y_i y_j\\
                  &= (\frac{1}{2}(x_i^{(1)} + x_i^{(2)})) (\frac{1}{2}(x_j^{(1)} + x_j^{(2)}))\\
                  &= \frac{1}{4} x_i^{(1)}x_j^{(1)} + \frac{1}{4} x_i^{(2)}x_j^{(1)} + \frac{1}{4} x_i^{(1)}x_j^{(2)} + \frac{1}{4} x_i^{(2)}x_j^{(2)}
\end{align*}
Comparing $q_{ij,BiDist}$ to $q_{ij,\Bi}$, we see that $q_{ij,BiDist}$ is a bilinear similarity with $\woo = \wto = \wot = \wtt = \frac{1}{4}$.

\subsubsection{Radial Distance}
The pairwise distance defined for the Radial Distance simulated data set is neither a squared Mahalanobis distance nor a Bilinear distance. Therefore, neither the \Ma~nor the \Bi~algorithms possess the capacity to learn such a distance from data. However, \Smerf~is a nonparametric, locally adaptive algorithm, and therefore has the expressive capacity for learning such a distance (in addition to being able to learn squared Mahalanobis and bilinear distances).

We recall that each $\mathbf{x}_i$ is uniformly distributed within the 20-dimensional centered unit ball. We defined the distance between $\mathbf{x}_i$ and $\mathbf{x}_j$ as
\begin{align}
    \zij = (|\mathbf{x}_i^{(1:2)}| - |\mathbf{x}_j^{(1:2)}|)^2
\end{align}
This distance can be seen as the squared difference between points along a one-dimensional manifold, where the manifold is a line segment from the center to the boundary of the 2-dimensional unit disk. A graphical depiction of this data set is shown in Figure \ref{fig: radial-distance}.

\begin{figure}[t]
\centering
\centerline{\includegraphics[width=0.5\columnwidth,trim={0in 0in 0 0},clip]{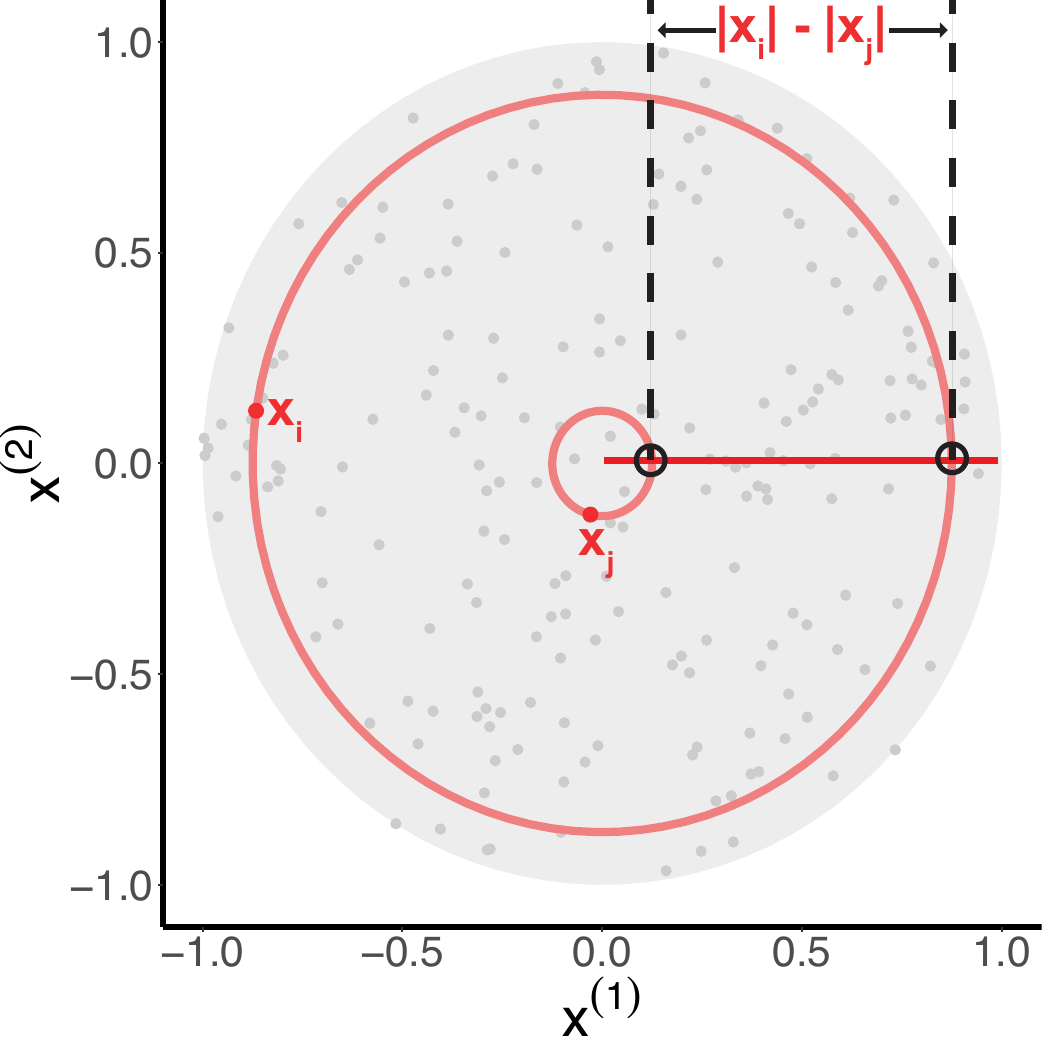}}
\caption{Graphical depiction of the first two dimensions of the Radial Distance data set. The distance between $\mathbf{x}_i$ and $\mathbf{x}_j$ is defined as $\zij = (|\mathbf{x}_i^{(1:2)}| - |\mathbf{x}_j^{(1:2)}|)^2$.}
\label{fig: radial-distance}
\end{figure}

\subsection{Details of Algorithm Implementation and Hyperparameters}
\label{sec: sim-hyper}
\Smerf~was implemented using in-house native R code. The number of trees for all experiments was 500. Each tree was trained on a random bootstrap sample of the original data. Recall that $p$ is the dimensionality of the input space. Two hyperparameters were tuned (the values tried for each hyperparameter are in parentheses):
\begin{itemize}
    \item \dd: the number of random dimensions to try splitting on at each split node ($p^{1/4}, p^{1/2}, p^{3/4}, p, p^{3/2}$)\\
    \item \mm: the minimum number of training points that must be in a node in order to attempt splitting; this is a stopping criteria ($2,4,8$)
\end{itemize}

The best pair of hyperparameter values was selected based on the out-of-bag RMSE. The trained forest using this selected pair of hyperparameter values was then used for prediction on the test set. We note that it is common for classification and regression forests to use the out-of-bag error as a proxy for cross-validation error.

\Ma~ and \Bi~ were implemented in Matlab. The optimization problems described for \Ma~and \Bi~in Section \ref{sec: experiments-simulations} were solved using the CVX convex optimization software package. Symmetric and positive semidefinite constraints were placed on $\Wm$ for both \Ma~and \Bi. The Mosek solver was used, with default solver settings.

\subsection{Estimation of Feature Importance on the Radial Distance Data Set}
\label{sec: feature-importance}

Importances of each of the 20 dimensions for the Radial Distance data set were estimated using a generalization of the standard Gini importance for classification \cite{randomforest}. Specifically, for a given feature, the importance is defined as the sum of the maximized objective function in Eq.~\ref{eq: opt-smerf} over all splits in the forest made on that particular feature. Thus, how frequently a split is made on that feature, as well as how extensively a split made on that feature reduces the average pairwise distance, dictate its importance estimate. Figure \ref{fig: feature-importance} shows the feature importance estimates (normalized by the maximum) from \Smerf~for the Radial Distance data set with 320 training examples. \Smerf~correctly identifies the first two dimensions.

\begin{figure}
\centering
\centerline{\includegraphics[width=0.8\columnwidth,trim={0in 0in 0 0},clip]{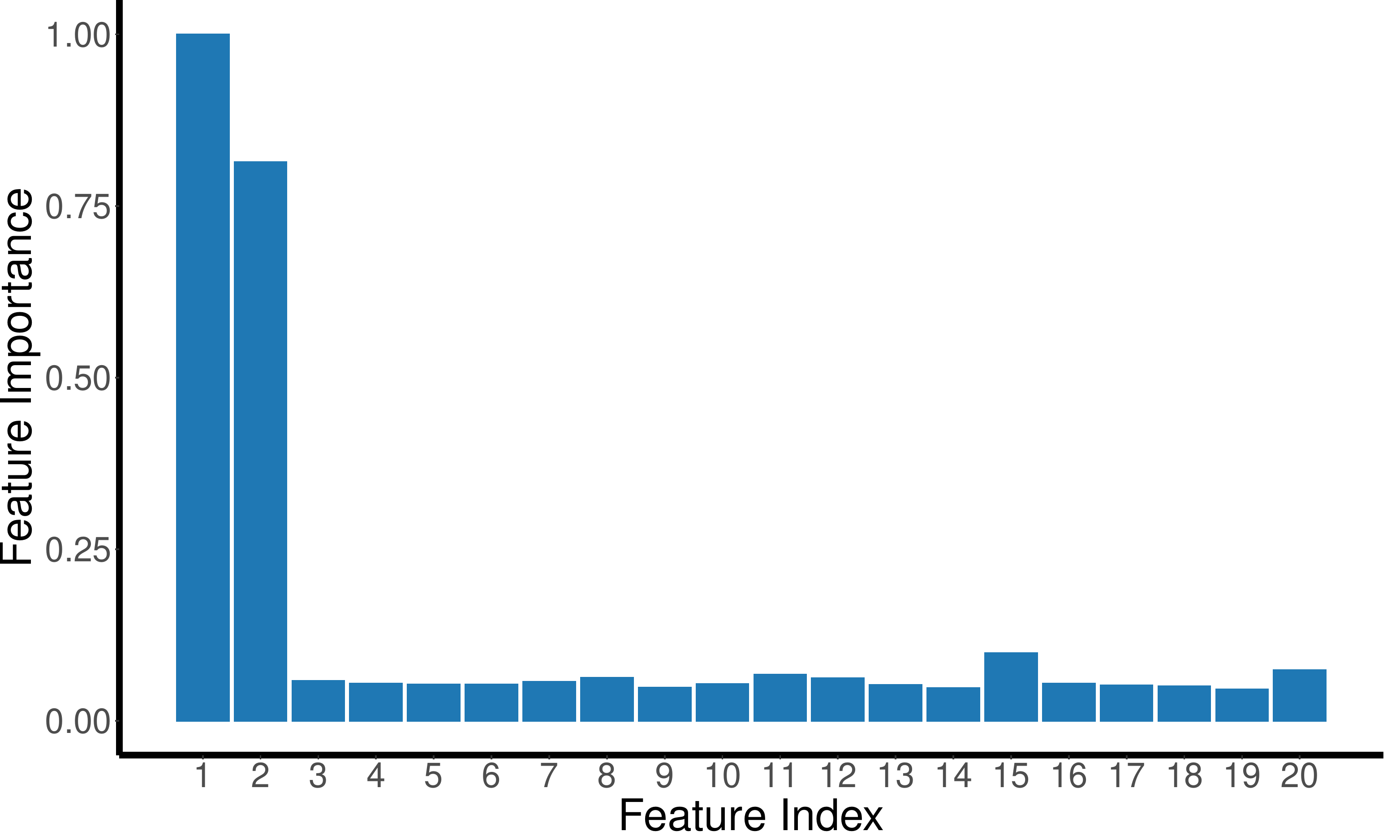}}
\caption{\Smerf~correctly identifies the first two dimensions as the important dimensions for the Radial Distance data set ($n = 320$).}
\label{fig: feature-importance}
\end{figure}

\section{Supplementary Information on Experiments in Section \ref{sec: experiments-networks}}
\label{sec: network-supplementary}

\subsection{Network Data Sets}
\label{sec: network-datasets}
We used the Lazega-cowork, Facebook-ego, and NIPS234 data sets from the github page for \Narm~\cite{zhao2017} (\url{https://github.com/ethanhezhao/NARM}). The Lazega-cowork data set used in \cite{zhao2017} is different from the original Lazega-cowork data set (\url{https://www.stats.ox.ac.uk/~snijders/siena/Lazega_lawyers_data.htm}). The original data set has a combined total of eight binary and ordinal node attributes. Since \Narm~and \Epm~can only handle binary attributes, \cite{zhao2017} encode the ordinal attributes into an expanded set of binary attributes, resulting in 18 binary attributes. We run \Smerf~on the original eight-attribute Lazega-cowork data set because \Smerf~can handle ordinal features, which is a benefit of our method.

\subsection{Training and Testing Partitioning}
As noted in Section \ref{sec: experiments-networks}, the proportion of data used for training was varied from 0.1 to 0.9 by an increment of 0.1. The word \emph{data} here has slightly different meanings for \Epm, compared to \Smerf~and \Narm. For \Smerf~and \Narm, the data was split into a set of training and testing \emph{nodes}. We noted that \Epm~fails at predicting links for entirely new nodes that were not seen at training because the model does not utilize node-attribute information. This is known as the \emph{cold-start} problem in relational learning and collaborative filtering. Therefore, for \Epm, the data was split into training and testing \emph{node pairs}, rather than training and testing \emph{nodes}. Splitting by node pairs is the partitioning scheme used by the authors of \Epm~\cite{zhou2015} and \Narm~\cite{zhao2017}. Let's refer to the \emph{node}-wise partitioning scheme as scheme-1, and the \emph{node pair}-wise partitioning scheme as scheme-2. We note that \Smerf~is incompatible with partitioning scheme-2 because this amounts to observing an adjacency matrix with missing entries, which \Smerf~currently cannot handle; handling of missing entries in adjacency/distance matrices is under development.

Because the task is link prediction, which is a \emph{pairwise} prediction, we wanted to keep the number of training \emph{node pairs} the same for both data partitioning schemes. Note that if there are $n$ nodes, then there are $(n^2 - n)/2$ node pairs. Denote by $f$ the proportion of nodes used for training using scheme-1. Thus if we sample $fn$ training nodes using scheme-1, then we sample $(fn)^2 - fn)/2$ training node pairs using scheme-2. The x-axis in Figure \ref{fig: net} represents $f$.


\subsection{Details of Algorithm Implementation and Hyperparameters}
\Narm~and \Epm~implementations were those hosted in the Github repository for \Narm.

The hyperparameters tuned in \Smerf~were the same as described in Section \ref{sec: sim-hyper}, except the hyperparameter $\dd \in \{p^{1/4}, p^{1/2}, p^{3/4}, p\}$. Selection of the best hyperparameter for AUC-ROC was based on the out-of-bag AUC-ROC, while selection of the bet hyperparameter for AUC-PR was based on the out-of-bag AUC-PR.

\Narm~and \Epm~have one hyperparameter: the truncation level $K_{max}$. The value for both algorithms was 50, 100, and 100 for the Lazega-cowork, Facebook-ego, and NIPS234 data sets, respectively. These were the settings noted in \cite{zhao2017} as being best.